%% file: paper.tex
\documentclass{article}

\usepackage{iclr2027_conference,times}
\iclrfinalcopy

\input{math_commands.tex}

\usepackage[utf8]{inputenc} 
\usepackage[T1]{fontenc}    
\usepackage{url}            
\usepackage{booktabs}       
\usepackage{amsfonts}       
\usepackage{nicefrac}       
\usepackage{microtype}      
\usepackage{xcolor}         

\usepackage{graphicx}
\usepackage{subfigure}
\usepackage{amsmath}
\usepackage{amsthm}
\usepackage{bm}
\usepackage{bbm}
\usepackage{wrapfig}
\usepackage{tcolorbox}

\usepackage{tikz}
\usepackage{pgfplots}
\pgfplotsset{compat=1.18}
\usetikzlibrary{arrows.meta,positioning,fit,backgrounds}

\usepackage{algorithm}
\usepackage{algorithmic}
\usepackage{graphicx}
\usepackage{textcomp}
\usepackage{subfigure} 
\usepackage{placeins}
\usepackage{multirow}
\usepackage{multicol}
\usepackage{xcolor, colortbl}
\usepackage{bm}
\usepackage{bbm}
\usepackage[normalem]{ulem}
\usepackage{enumitem}
\usepackage{caption}
\usepackage{nicefrac}
\usepackage{amssymb}
\usepackage{makecell}

\newtheorem{lemma}{Lemma}
\newtheorem{proposition}{Proposition}
\newtheorem{corollary}{Corollary}
\newtheorem{definition}{Definition}

\newcommand{\method}{{GradCodeS}}

\newcommand{\stat}[2]{\ensuremath{#1_{\scriptscriptstyle #2}}}
\newcommand{\beststat}[2]{\ensuremath{\mathbf{#1}_{\scriptscriptstyle #2}}}
\newcommand{\secondstat}[2]{\ensuremath{\underline{#1}_{\scriptscriptstyle #2}}}

\usepackage{array}
\newcolumntype{L}[1]{>{\raggedright\let\newline\\\arraybackslash\hspace{0pt}}m{#1}}
\newcolumntype{C}[1]{>{\centering\let\newline  \\\arraybackslash\hspace{0pt}}m{#1}}
\newcolumntype{R}[1]{>{\raggedleft\let\newline \\\arraybackslash\hspace{0pt}}m{#1}}

\renewcommand{\eqref}[1]{(\ref{#1})}

\title{Fine-Tuning Low-Bit Models with Gradient in Quantized Code Space}

\author{
Shiguang Wu\textsuperscript{1},
Zhouchen Lin\textsuperscript{2},
Quanming Yao\textsuperscript{1}\\
\textsuperscript{1}Department of Electronic Engineering, Tsinghua University\\
\textsuperscript{2}School of Intelligence Science and Technology, Peking University\\
\texttt{wsg23@mails.tsinghua.edu.cn, qyaoaa@tsinghua.edu.cn}
}

\begin{document}

\maketitle	
\lhead{}

\begin{abstract}
Fine-tuning Low-bit models aims to adapt a quantized model while keeping the final deployed checkpoint in the same low-bit form. This setting is practically important as it reduces memory and inference cost for storage and deployment.
Under this constraint, adaptation becomes an optimization problem over quantization codes and scales. 
Existing
continuous low-bit training is efficient, but it can be distorted by straight through estimation error or by post-quantize gap; discrete search is deployment-faithful, but it is often too inefficient under a finite training budget.
We propose code surrogate gradient as the first order signal in deployable code space to acceleate optimization, and performing guided search to preserve deployment faithfulness.
Experiments across arithmetic reasoning, instruction following, and structured language understanding show that \method{} consistently improves fine-tuning low-bit models across different quantization datatypes. Code is provided at \url{https://github.com/ovo67/GradCodes}.
\end{abstract}


\section{Introduction}

Large language models (LLMs) are increasingly deployed in low-precision formats such as INT4, NF4, and MXFP4 to reduce storage, memory traffic, and inference cost~\citep{xiao2023smoothquant,lin2024awq}. These savings matter not only for a shared backbone, but also when many task- or user-specific checkpoints must be stored under a fixed memory budget. We study a strict fully low-bit setting in which the targeted weights of every adapted checkpoint remain in the target quantized representation, with only datatype-permitted scale metadata and no high-precision residual adapter at inference time~\citep{jeon2025l4q,malinovskii2024pvtuning}. Under this constraint, fine-tuning is no longer optimization in a nearly continuous floating-point space; it is optimization over quantization codes and scales.

One practical strategy retains continuous optimization. Quantized-backbone adapter methods train high-precision LoRA-style parameters and either keep them at inference time or merge and requantize them afterward~\citep{dettmers2023qlora,li2023loftq,loeschcke2024loqt,xu2023qalora}. This is effective when mixed-precision deployment or a separate conversion stage is acceptable. Under the stricter setting above, however, retaining the adapter violates the deployment constraint, while merging and requantizing can make the evaluated low-bit state differ from the state optimized during training. Quantization-aware and compressed-representation methods move optimization closer to the final checkpoint~\citep{jeon2025l4q,malinovskii2024pvtuning}, but still leave a fundamental question: how should first-order information be expressed when the actual decision variable is a quantization-code index?

Direct discrete and zeroth-order approaches instead evaluate candidates in or near the quantized model family~\citep{zhou2025quzo,shang2025qzo,xu2026qes}, but random or population proposals can require many forward evaluations in high-dimensional code spaces. The central challenge is to combine deployment-faithful selection with first-order directional guidance.

To be deployment-faithful, we optimize over quantized code coordinates rather than weight. To be efficient, we try to find more useful information beyond zero-order.
Based on this principle, we propose \method, a gradient-guided search method. We first
identify that mapping ordinary gradient in weight space to code space is generally not the steepest descent direction in code space, as illustrated in Figure~\ref{fig:spaces}, thus introducing a \emph{code-coordinate surrogate gradient} by codebook continuation, serving as first-order information in code space with theoretical support.
\method{} uses code-coordinate surrogate gradient descent to provide reference and then  samples multiple nearby code candidates biased toward that reference and selects an update only according to its realized deployed loss. Thus every selected iterate is a valid low-bit checkpoint.
This rule is parameterization-agnostic and datatype-agnostic: it can be instantiated on full weights, integrated with structured parameterizations such as LoRA, combined with joint scale optimization, and paired with different quantization datatypes.

%

\begin{figure}[t]
	\vspace{-5pt}
	\centering
	\subfigure[Geometry mismatch between weight and code spaces.\label{fig:spaces}]{
		\includegraphics[width=0.50\linewidth]{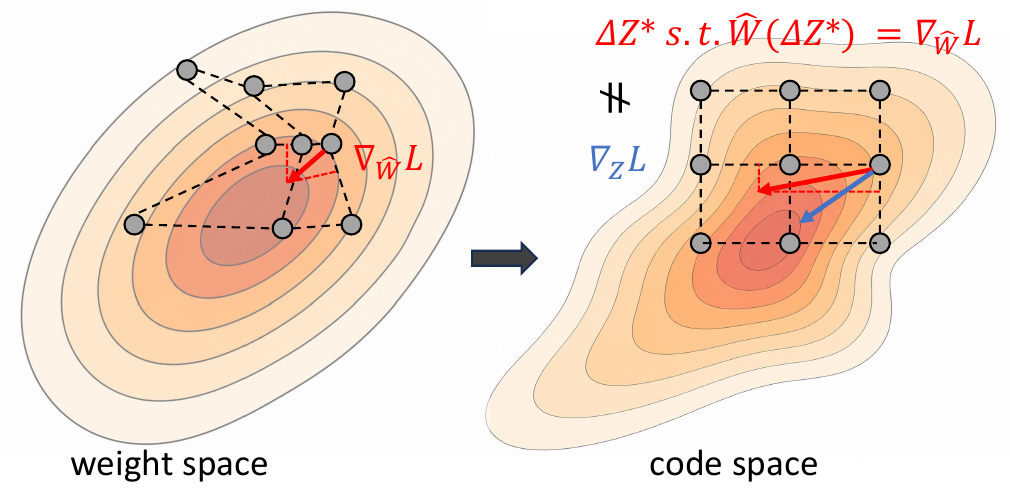}}
	\hfill
	\subfigure[Accuracy under different deployment regimes.\label{fig:performance}]{
		\includegraphics[width=0.48\textwidth]{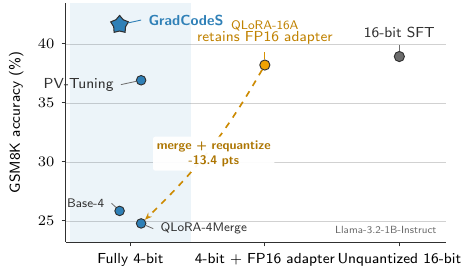}}
	\caption{Motivation for deployment-faithful code-space optimization. (a) Group scales and non-uniform codebook gaps distort local geometry, so a weight gradient does not identify the best discrete code move. (b) In the Llama-3.2-1B GSM8K example, continuous adaptation either retains an FP16 adapter or loses 13.4 points after QLoRA requantization; \method{} improves the fully 4-bit result while retaining a deployable checkpoint.}
	\label{fig:motivation}
\end{figure}

Our contributions can be summarized as follows:
\begin{itemize}[leftmargin=*]
	\item We derive a \emph{code-coordinate surrogate gradient} and establish its local steepest-descent interpretation under codebook continuation.
	
	\item We develop a guided candidate-sampling algorithm that combines first-order guidance with loss-based selection over deployable low-bit candidates.
	
	\item We evaluate \method{} across three task families, multiple quantization datatypes, and both full and structured code parameterizations, observing consistent gains in the fully low-bit setting.
\end{itemize}

\section{Related Work}

\textbf{Continuous low-bit adaptation.} Low-bit adaptation commonly retains continuous optimization. Quantized-backbone parameter-efficient methods train high-precision adapters, improve their quantized initialization, or design updates that can later be merged~\citep{dettmers2023qlora,li2023loftq,xu2023qalora,loeschcke2024loqt}. They are effective when mixed-precision inference or a separate conversion stage is acceptable, but under strict fully low-bit deployment the optimized and deployed states may differ because an adapter remains at inference time or a merged update must be requantized. Quantization-aware and compressed-representation methods, including L4Q and PV-Tuning, optimize closer to the final quantized checkpoint~\citep{jeon2025l4q,malinovskii2024pvtuning}. Their focus is efficient optimization of differentiable or structured compressed representations. Our complementary question is how a weight gradient should be expressed in ordered code coordinates without treating the resulting continuous reference as an accepted update.

\textbf{Direct low-bit optimization.} This approach avoids a separate deployment conversion. QuZO and QZO use zeroth-order estimates for low-precision or quantization-aware fine-tuning, while QES applies evolution strategies to quantized parameters~\citep{zhou2025quzo,shang2025qzo,xu2026qes}. These methods reduce or avoid backpropagation and evaluate search information through forward perturbations or population fitness. Although candidate quality can therefore be measured in or near the quantized model family, proposal efficiency remains difficult in a high-dimensional space under a finite evaluation budget. \method{} instead spends one backward pass to obtain an analytical local signal, uses it only to guide where discrete search looks, and still selects updates by the realized loss of deployable low-bit candidates. 

\textbf{Gradient-informed discrete proposals.} A derivative can shape a discrete proposal without being followed deterministically. Representative samplers construct local, Langevin-style, Newton-style, or importance proposals from gradients with respect to discrete inputs~\citep{grathwohl2021oops,zhang2022langevin,sun2023dlmc,xiang2023newton,liu2022gradient}; related optimization methods bias random search with surrogate gradients or adapt discrete sampling to combinatorial objectives~\citep{maheswaranathan2019guided,sun2023revisiting,li2025reheated}. These works motivate separating proposal construction from state selection, but they do not resolve two features of low-bit fine-tuning. First, the available derivative is with respect to recovered weights, whereas the decision variable is a code index whose geometry is distorted by group scales and non-uniform codebook gaps. Second, our goal is loss minimization over a finite candidate set rather than preservation of a stationary distribution. \method{} addresses these differences through a quantization-specific code-coordinate signal, proposal-only gradient guidance, and loss-based selection of deployable states.


\section{Problem Setup}
\label{sec:preliminaries}

Fix a target quantization datatype $q$, specified by its bitwidth
$b_q$, numerical codebook $\mathcal C_q$, grouping structure
$\mathcal G_q$, and admissible scale space $\mathcal S_q$. Its ordered
codebook is
\begin{equation*}
	\mathcal C_q
	=\{c^{(q)}_1,\ldots,c^{(q)}_{K_q}\},
	\qquad
	c^{(q)}_1\leq\cdots\leq c^{(q)}_{K_q},
	\qquad K_q\leq 2^{b_q}.
\end{equation*}
The indices follow numerical order and need not match hardware bit-pattern
order. For a weight matrix
$W\in\mathbb R^{d_{\mathrm{out}}\times d_{\mathrm{in}}}$, let
$[K_q]:=\{1,\ldots,K_q\}$ and
$\mathcal Z_q:=[K_q]^{d_{\mathrm{out}}\times d_{\mathrm{in}}}$.
A code matrix $Z\in\mathcal Z_q$ selects values entrywise as
$\mathcal C_q(Z)=[c^{(q)}_{Z_{ij}}]_{ij}$, where $\big[\psi(i,j)\big]_{ij}$ denotes the matrix whose $(i,j)$-th element is $\psi(i,j)$. Each entry belongs to a group
$g_q(i,j)\in\mathcal G_q$ with scale metadata
$s=\{s_g\}_{g\in\mathcal G_q}\in\mathcal S_q$. Defining
$S_q(s)=[s_{g_q(i,j)}]_{ij}$, the deployed weight is
\begin{equation}
	\widehat W_q(Z,s)
	=
	S_q(s)\odot\mathcal C_q(Z),
	\label{eq:recover}
\end{equation}
where $\odot$ is element-wise multiplication. This representation covers the evaluated datatypes
NF4, INT4, and MXFP4; their datatype-specific components are given
in Appendix~\ref{app:datatype}.

For a model, let $\mathcal L$ be the set of targeted quantized linear
layers. We use $Z$ and $s$ for their layer-wise collections and let
$\mathcal Q_q$ denote the product of the corresponding layer-wise
feasible code and scale sets. We use
$\widehat{\mathbf W}_q(Z,s)$ for the recovered weights. 
Parameters outside
$\mathcal L$ are left unquantized.

Given an initial feasible state $(Z_0,s_0)$, a target distribution
$\mathcal D$, a model $f$, and a loss function $\ell$, we formulate
fully low-bit fine-tuning as
\begin{equation}
\min\nolimits_{(Z,s)\in\mathcal Q_q}
L(Z,s)
:=
\mathbb E_{(x,y)\sim\mathcal D}
\left[
\ell\!\left(
f\!\left(x;\widehat{\mathbf W}_q(Z,s)\right),y
\right)
\right].
\label{eq:general-objective}
\end{equation}
Unlike approaches that optimize a continuous auxiliary state and
quantize only after training, \eqref{eq:general-objective} is defined
directly over deployable states. Here $Z$ is discrete, while $s$ is
continuous in the default setting but constrained to its admissible
space. Since $q$ is fixed, we henceforth omit datatype annotations from
$K_q$, $c_k^{(q)}$, $\mathcal C_q$, $\mathcal Z_q$, $\mathcal S_q$, and
$S_q$. Thus, for a single matrix,
$\widehat W(Z,s)=S(s)\odot\mathcal C(Z)$. Bold $\widehat{\mathbf W}$ is
reserved for the full collection of targeted weights, while $L$ denotes
the corresponding objective. 


\section{Method}
\label{sec:metitems}

%

To optimize \eqref{eq:general-objective}, \method{} separates gradient
guidance from discrete state selection. Section~\ref{sec:code-gradient}
constructs a geometry-aware code surrogate gradient, and
Section~\ref{sec:code-properties} establishes its local and discrete
first-order properties. Section~\ref{sec:guided} then incorporates this
signal into a \emph{guide--sample--evaluate--select} procedure that
accepts updates according to the realized loss of deployable low-bit
states.

\subsection{Constructing the Code Surrogate Gradient}
\label{sec:code-gradient}


The ordinary weight gradient $\nabla_{\widehat W}L$ is not directly a signal for the discrete code index because group scales and local codebook gaps rescale a unit code move. Thus in general, the direction in code space corresponding to the steepest descent direction in weight space is \textbf{not} the steepest descent direction in code space, as illustrated in Figure~\ref{fig:def1-illustration}.

For a matrix $A$, write $A_+=[\max\{A_{ij},0\}]_{ij}$ and $A_-=[\max\{-A_{ij},0\}]_{ij}$.
Let $Z^\pm=\mathrm{clip}(Z\pm\mathbbm{1},1,K)$ and define the forward and backward gaps
\[
\begin{aligned}
G^+(Z)&=\bigl[c_{Z^+_{ij}}-c_{Z_{ij}}\bigr]_{ij},
\quad
G^-(Z)&=\bigl[c_{Z_{ij}}-c_{Z^-_{ij}}\bigr]_{ij}.
\end{aligned}
\]
The quantities above characterize the local geometry of the two neighboring code moves. To use a weight gradient for discrete search, we need a code-coordinate signal that reflects this geometry; Definition~1 constructs such a signal.

\begin{definition}[Local effective step and code surrogate gradient]\label{def:sg}
	Given a quantized state $(Z,s)$ with fixed codebook $\mathcal C$, let $\sigma=\operatorname{sign}(\nabla_{\widehat W}L(Z,s))$ and define the descent-aligned effective step by
	$D\bigl(Z,s,\nabla_{\widehat W}L(Z,s)\bigr)
		=S(s)\odot\bigl(G^+(Z)\odot\sigma_-+G^-(Z)\odot\sigma_+\bigr)$.
	The directional code surrogate gradient is then
	\begin{align}
		\nabla_Z L(Z,s)
		=
		\nabla_{\widehat W}L(Z,s)
		\odot
		D\bigl(Z,s,\nabla_{\widehat W}L(Z,s)\bigr).
		\label{eq:glat}
	\end{align}
\end{definition}

Figure~\ref{fig:def1-illustration} gives a simplified illustration of Definition~\ref{def:sg}. It uses an identity codebook, whose adjacent gaps are all one, so the effective step reduces to $D=S(s)$. The example therefore shows
even without nonuniform codebook gaps, accounting for the deployed-weight change of each code coordinate yields a substantially better descent direction. 
Appendix~\ref{app:code-gradient-derivation} derives this construction from the local piecewise-linear continuation.

\begin{figure}[ht]
\centering
\begin{tcolorbox}[
  colback=black!1,
  colframe=black!22,
  boxrule=0.5pt,
  arc=1.5pt,
  left=5pt,right=5pt,top=2pt,bottom=2pt,
  before skip=0pt,after skip=0pt]
\small
Let $\mathcal C(Z)=Z$, $S(s)=(0.1,10)$, $L(\hat W)=\hat W_1+0.6\hat W_2$. $\nabla_{\widehat W}L=(1,0.6)$, so mapped code move (to achieve $\nabla_{\widehat W}L$) is $(-10,-0.06)$;  while $D=S(s)$ and $\nabla_ZL=(0.1,6)$. For an equal code-space step $\Delta Z=\epsilon d$ with $\|d\|_2=1$:
\par\vspace{2pt}
\renewcommand{\arraystretch}{1.18}
\begin{tabular*}{\linewidth}{@{\extracolsep{\fill}}lcc@{}}
\toprule
Direction & Unit direction $d$ & Loss change $\Delta L$ \\
\midrule
\textcolor{blue!65!black}{Mapped weight gradient}
  & $\displaystyle{(-10,-0.06)}/{\sqrt{100.0036}}$ & $-0.136\epsilon$ \\
\textcolor{red!70!black}{Definition~1 surrogate}
  & $\displaystyle{(-0.1,-6)}/{\sqrt{36.01}}$ & $-6.001\epsilon$ \\
\bottomrule
\end{tabular*}
\noindent\textbf{Result.} The effective-step correction yields about $44\times$ more loss decrease for the same code-space step.
\end{tcolorbox}
\vspace{5pt}
\caption{A two-coordinate illustration of the effective-step correction in Definition~1.}
\label{fig:def1-illustration}
\vspace{-20pt}
\end{figure}

\subsection{First-Order Properties}
\label{sec:code-properties}

The results below summarize the overall argument. Lemma~1 gives the surrogate a local continuous interpretation. Proposition~1 establishes the central discrete connection by showing exact first-order consistency for feasible one-level code moves. Corollary~1 then converts this identity into a realized-loss bound used by the candidate-search analysis in Proposition~2.

\begin{lemma}[Local steepest-descent interpretation]
	\label{lem:code-local-surrogate}
	Let $D_{Z,s}:=D\bigl(Z,s,\nabla_{\widehat W}L(Z,s)\bigr)$. Define the local directional surrogate function of $L(Z,s)$ by extending it to real-valued interpolation $\Delta Z\in\mathbb R^{d_{\mathrm{out}}\times d_{\mathrm{in}}}$ by
$$
	\bar{L}_{Z,s}(\Delta Z)
	=
	L\!\left(
	\widehat W(Z,s)
	+
	D_{Z,s}\odot\Delta Z
	\right).
$$
	Then
$
	\left.
	\nabla_{\Delta Z}
	\bar{L}_{Z,s}(\Delta Z)
	\right|_{\Delta Z=0}
	=
	\nabla_Z L(Z,s).
	$
	Consequently, whenever $\nabla_Z L(Z,s)\neq 0$, the direction $-\nabla_Z L(Z,s)$ is the steepest descent direction of $\bar{L}_{Z,s}$ at $\Delta Z=0$ under the Frobenius norm.
\end{lemma}
Lemma~\ref{lem:code-local-surrogate} gives $-\nabla_ZL$ a precise but deliberately local interpretation: it is the steepest direction of the active continuous surrogate, not an automatically valid code update. We next connect its score to an actual one-level code move.

\begin{proposition}[One-step consistency]
	\label{prop:code-one-step-consistency}
	Let $\Delta Z\in\{-1,0,1\}^{d_{\mathrm{out}}\times d_{\mathrm{in}}}$ be a feasible one-step code move such that
$
	Z+\Delta Z\in\mathcal Z
$,
$
	\Delta Z\odot\nabla_{\widehat W}L(Z,s)\le 0
$
entrywise.
	Then the surrogate linear term matches the exact first-order deployed-weight change:
	\begin{equation}
		\begin{aligned}
			\langle \nabla_Z L(Z,s),\Delta Z\rangle_F
			&=
			\left\langle
			\nabla_{\widehat W}L(Z,s),
			\widehat W(Z+\Delta Z,s)-\widehat W(Z,s)
			\right\rangle_F .
		\end{aligned}
		\label{eq:one-step-consistency}
	\end{equation}
\end{proposition}
Proposition~\ref{prop:code-one-step-consistency} provides the central discrete connection: by selecting the adjacent gap in the descent direction, the surrogate score exactly matches the first-order deployed-weight change of every feasible gradient-aligned one-level move. Because the network loss is nonlinear, the next result controls the discrepancy from the realized loss.

\begin{corollary}[One-step smooth-loss upper bound]
	\label{prop:code-one-step-upper-bound}
	Under the assumptions of Proposition~\ref{prop:code-one-step-consistency}, if $L$ is $\beta$-smooth as a function of the deployed weight in a neighborhood of $\widehat W(Z,s)$, then
	\[
	L(Z+\Delta Z,s)
	\le
	L(Z,s)
	+
	\langle \nabla_Z L(Z,s),\Delta Z\rangle_F
	+
	\frac{\beta}{2}
	\left\|
	\widehat W(Z+\Delta Z,s)-\widehat W(Z,s)
	\right\|_F^2 .
	\]
\end{corollary}
Corollary~\ref{prop:code-one-step-upper-bound} bounds the realized loss change by the surrogate prediction plus a curvature penalty; a move is guaranteed to decrease the loss when its negative surrogate term dominates this penalty. Thus the signal is first-order faithful for feasible local moves, while candidate evaluation accounts for nonlinear and higher-order effects. This bound underlies Proposition~\ref{prop:expected-code-descent}; proofs of the present results are provided in Appendix~\ref{app:code-gradient}.

\subsection{Gradient-Guided Discrete Optimization}
\label{sec:guided}
The results in Section~\ref{sec:code-properties} show how to rank feasible local code moves. We now turn this signal into a complete procedure that alternates differentiable scale updates with deployment-faithful code search. The surrogate gradient guides where to search, while the deployed loss determines what to accept.

Performing gradient descent with the code surrogate gradient~\eqref{eq:glat} gives a real-valued reference point in code space:
\begin{equation}
	Z^{\mathrm{ref}} = Z - \eta_Z \nabla_Z L(Z,s),
	\label{eq:guided-reference-z}
\end{equation}
where $\eta_Z>0$ is the step size. Because $Z^{\mathrm{ref}}$ is generally non-integer, it only centers a proposal over valid codes. The \emph{guide--sample--evaluate--select} pipeline draws $M\ll|\mathcal Z|$ candidates with greater mass near $Z^{\mathrm{ref}}$ and retains the lowest-loss state. Appendix~\ref{app:proposal-spec} gives an efficiently sampled coordinate-factorized proposal.

In practice, $\widehat L_{\mathcal B}(Z,s)$ denotes the average loss over mini-batch $\mathcal B$, an unbiased estimator of $L(Z,s)$. Each iteration first updates the differentiable scales by projected gradient descent with $Z$ fixed and then searches the codes at the updated scales. Algorithm~\ref{alg:code-search} summarizes the procedure; $\Pi_{\mathcal S}$ projects onto the admissible scale space, and selection ties retain the current $Z$.

\begin{algorithm}[t]
	\caption{\method: Guide--Sample--Evaluate--Select in Quantized Code Space}
	\label{alg:code-search}
	\begin{algorithmic}[1]
		\STATE \textbf{Input:} state $(Z,s)$; codebook $\mathcal C$; data $\mathcal D$; feasible sets $(\mathcal Z,\mathcal S)$; step sizes $(\eta_Z,\eta_s)$; proposal $p$; candidates $M$; iterations $T$.
		\FOR{$t=1,\ldots,T$}
		\STATE Sample mini-batch $\mathcal B_t\sim\mathcal D$
		\STATE Update scales by $s\leftarrow \Pi_{\mathcal S}(s-\eta_s\nabla_{s} \widehat L_{\mathcal B_t}(Z,s))$
		\STATE \textit{Guide:} compute $\nabla_Z\widehat L_{\mathcal B_t}(Z,s)$ using \eqref{eq:glat}, then $Z^{\mathrm{ref}}$ using \eqref{eq:guided-reference-z}
		\STATE \textit{Sample:} sample $M$ code candidates $\{Z^m\}_{m=1}^M$ from $p(\cdot\mid Z^{\mathrm{ref}})$
		\STATE \textit{Evaluate:} compute $\widehat L_{\mathcal B_t}(Z^m,s)$ for each candidate
		\STATE \textit{Select:} $Z\leftarrow\arg\min_{\widetilde Z\in\{Z,Z^1,\ldots,Z^M\}}\widehat L_{\mathcal B_t}(\widetilde Z,s)$
		\ENDFOR
		\STATE \textbf{Return:} $\widehat W(Z,s)=S(s)\odot\mathcal C(Z)$
	\end{algorithmic}
\end{algorithm}

The next result isolates the role of gradient guidance in the deterministic full-batch code-search substep; it is conditional on the proposal mass placed near the reference step.

\begin{proposition}[Guided best-of-$M$ descent]
	\label{prop:expected-code-descent}
	Fix the scales $s$, let $q_t$ be the proposal mass of gradient-aligned one-step feasible updates $\Delta$ satisfying $\|\Delta+\eta_Z\nabla_ZL(Z_t,s)\|_F\le c\eta_Z\|\nabla_ZL(Z_t,s)\|_F$, and suppose that Corollary~\ref{prop:code-one-step-upper-bound} applies with $\|\widehat W(Z_t+\Delta,s)-\widehat W(Z_t,s)\|_F\le\kappa\|\Delta\|_F$. If $c\in[0,1)$ and $\alpha:=1-c-\frac{\beta\kappa^2\eta_Z}{2}(1+c)^2>0$, then $M$ independent candidates followed by the selection rule in Algorithm~\ref{alg:code-search} satisfy
	$	\mathbb E\!\left[L(Z_{t+1},s)\mid Z_t\right]
	\le
	L(Z_t,s)
	-
	\alpha\eta_Z\bigl[1-(1-q_t)^M\bigr]\|\nabla_ZL(Z_t,s)\|_F^2$.
\end{proposition}
The factor $1-(1-q_t)^M$ is the probability that at least one candidate lies in the specified near-reference set, showing how proposal concentration and candidate count control expected progress. The proof is provided in Appendix~\ref{app:expected-descent}.

\paragraph{Operational distinction.}
Continuous or STE-style approaches use a gradient to produce an update that is retained, projected, or requantized, whereas zeroth-order and evolutionary approaches preserve discrete evaluation but rely on sampled losses to determine where to search. \method{} separates these roles: a quantization-aware backward signal shapes a proposal over valid codes, while realized discrete losses select the accepted deployable state. Section~\ref{sec:experiments} evaluates the two practical consequences of this separation: deployment-faithful performance and more effective candidate search.

\section{Experiments}
\label{sec:experiments}

We evaluate \method{} to address three questions: whether \method{} improves performance while remaining in the fully 4-bit deployment space; whether the results are robust across data regimes, code parameterizations, and quantization datatypes; and whether gradient guidance improves candidate quality and end-to-end search efficiency.

\subsection{Experimental Settings}

We evaluate Qwen3-0.6B, Llama-3.2-1B and 3B-Instruct on GSM8K, AlpacaEval, and MASSIVE en-US, covering arithmetic reasoning, instruction following, and structured semantic parsing. We quantize all transformer-block linear layers while leaving the token embedding and LM head unquantized. Under matched data and evaluation protocols, we compare unquantized references (Base-16 and 16-bit SFT), mixed-precision adapter methods (QLoRA-16A and QA-LoRA), and fully 4-bit baselines (Base-4, QLoRA-4Merge, LoQT, L4Q, PV-Tuning, QuZO, QZO, and QES). We evaluate both full-matrix and low-rank code parameterizations of \method{}; complete dataset, metric, quantization, and baseline details are provided in Appendix~\ref{app:experimental-details}. All results are obtained with three independet runs with different random seeds, reported in $\text{mean}_\text{std}$.

\subsection{Main Results}
Tables~\ref{tab:main_gsm8k_alpaca} and~\ref{tab:main_massive} report results across arithmetic reasoning, instruction following, and structured semantic parsing. Across both backbones and all three tasks, \method{} is the strongest fully 4-bit method. Direct 4-bit quantization causes a substantial drop, and the competing fully 4-bit methods recover this loss to varying degrees; \method{} provides the most consistent recovery across the reported metrics.
We skip relatively weak baselines on 3B model due to computation budget.

The comparison between QLoRA-16A and QLoRA-4Merge isolates the deployment-state mismatch. QLoRA-16A retains a high-precision adapter at inference time, whereas QLoRA-4Merge collapses that adapter into the backbone and requantizes the resulting weights. The degradation after this conversion, also highlighted in Figure~\ref{fig:performance}, supports deployment-faithful optimizing rather than relying on an optimize-then-quantize pipeline. Within the fully 4-bit regime, both \method{} parameterizations are effective: the LoRA variant is generally strongest, while the full-matrix variant remains competitive. We examine their dependence on training-set size in Section~\ref{sec:size}.

Compared with 16-bit SFT, the best \method{} variant obtains higher GSM8K accuracy and stronger MASSIVE results, while the 16-bit reference remains strongest on Qwen AlpacaEval and is marginally higher on Llama AlpacaEval. We interpret the gains with low-bit codes as regularization in Section~\ref{sec:size}, not as a general claim that lower precision is universally superior.

\paragraph{Deployment cost view.}
Figure~\ref{fig:performance} complements the task tables by plotting Llama-3.2-1B-Instruct GSM8K accuracy across deployment regimes. The comparison is computed on the target quantized modules and excludes modules intentionally left unquantized. \method{} attains the best accuracy on the fully 4-bit deployment line, whereas mixed-precision methods retain additional high-precision adapter parameters.

\begin{table}[ht]
	\centering
	\small
		\caption{Results on GSM8K and AlpacaEval. Values are means over three independent runs, with standard deviations shown as subscripts. The second column reports the final deployment regime: 16 denotes unquantized 16-bit deployment, Mixed denotes a 4-bit backbone with high-precision adapter parameters, and 4 denotes fully 4-bit deployment for the quantized modules.}
	\label{tab:main_gsm8k_alpaca}
\resizebox{\linewidth}{!}{%
\begin{tabular}{lccccccc}
	\toprule
	\multirow{2}{*}{\textbf{Method}}
	& \multirow{2}{*}{\makecell{\textbf{Deployment}\\\textbf{bits}}}
	& \multicolumn{2}{c}{\textbf{Qwen3-0.6B}}
	& \multicolumn{2}{c}{\textbf{Llama-3.2-1B-Inst.}} 
	& \multicolumn{2}{c}{\textbf{Llama-3.2-3B-Inst.}}\\
	\cmidrule(lr){3-4}
	\cmidrule(lr){5-6}
	\cmidrule(lr){7-8}
	&
	& \makecell{\textbf{GSM8K}\\\textbf{Acc.}}
	& \makecell{\textbf{Alpaca}\\\textbf{WR}}
	& \makecell{\textbf{GSM8K}\\\textbf{Acc.}}
	& \makecell{\textbf{Alpaca}\\\textbf{WR}} 
	& \makecell{\textbf{GSM8K}\\\textbf{Acc.}}
	& \makecell{\textbf{Alpaca}\\\textbf{WR}}\\
	\midrule
	Base-16        & 16       & \stat{37.45}{0.00} & \textit{base}& \stat{36.32}{0.00} &\textit{base}& \stat{56.94}{0.00} &\textit{base} \\
	16-bit SFT    & 16        & \stat{39.65}{0.54} & \beststat{88.26}{2.16} & \stat{38.93}{0.63} & \beststat{56.64}{2.36} & \beststat{67.10}{0.28} & \beststat{78.57}{2.86} \\
	\midrule
	QLoRA-16A      & Mixed (4.08)     & \stat{39.57}{0.47} & \stat{54.90}{2.87} & \stat{38.21}{0.51} & \stat{53.89}{2.74} & \stat{65.05}{0.39} & \stat{68.26}{3.12} \\
	QA-LoRA        & Mixed (4.05)    & \stat{36.13}{0.62} & \stat{53.34}{3.12} & \stat{34.15}{0.74} & \stat{49.23}{3.18} & - & - \\
	\midrule\midrule
	Base-4         & \textbf4     & \stat{27.14}{0.00} & \stat{26.58}{1.94} & \stat{25.85}{0.00} & \stat{46.16}{2.42} & \stat{49.13}{0.00} & \stat{35.94}{2.68} \\
	QLoRA-4Merge       & \textbf4   & \stat{26.46}{0.71} & \stat{38.35}{2.76} & \stat{24.79}{0.66} & \stat{33.68}{3.55} & \stat{49.13}{0.52} & \stat{36.20}{3.07} \\
	LoQT           & \textbf4     & \stat{30.94}{0.58} & \stat{44.67}{2.63} & \stat{29.04}{0.61} & \stat{48.85}{2.81} & - & - \\
	L4Q            & \textbf4     & \stat{29.89}{0.67} & \stat{40.28}{2.91} & \stat{31.06}{0.57} & \stat{51.21}{2.66} & - & - \\
	PV-Tuning      & \textbf4 & \stat{38.70}{0.49} & \stat{70.51}{3.21} & \stat{36.92}{0.45} & \stat{52.63}{1.84}& \stat{63.08}{0.43} & \stat{62.36}{2.79}  \\
	QuZO           & \textbf4     & \stat{29.45}{0.84} & \stat{34.35}{3.47} & \stat{26.54}{0.78} & \stat{41.57}{3.26} & - & -  \\
	QZO            & \textbf4     & \stat{29.76}{0.76} & \stat{38.99}{3.14} & \stat{27.60}{0.81} & \stat{39.39}{3.41} & - & -  \\
	QES            & \textbf4     & \stat{28.80}{0.92} & \stat{41.13}{3.36} & \stat{26.39}{0.85} & \stat{38.12}{3.67} & - & -  \\ \midrule
	\textbf{\method{}} (LoRA)   & \textbf4     & \beststat{45.72}{0.42} & \stat{82.24}{2.57} & \secondstat{40.52}{0.47} & \stat{54.66}{1.83}& \secondstat{67.08}{0.31} & \stat{69.52}{3.26} \\
		\textbf{\method{}} (Full)       & \textbf4     & \secondstat{43.68}{0.38} & \secondstat{83.09}{2.48} & \beststat{41.63}{0.58} & \secondstat{56.34}{2.31} & \stat{66.39}{0.35} & \secondstat{72.04}{3.14}\\
	\bottomrule
\end{tabular}
}
\end{table}

\paragraph{Structured prediction.}

Table~\ref{tab:main_massive} evaluates a more application-like structured prediction setting. \method{} gives the best fully 4-bit results on both backbones and all three metrics. Exact Match requires the entire semantic frame to be correct, while Slot F1 measures the quality of the predicted slot structure. The gains on both metrics indicate that deployment-faithful code updates preserve structured output behavior after quantization.

\begin{table}[ht]
	\centering
	\setlength\tabcolsep{3pt}
	\small
		\caption{MASSIVE en-US test results. Values are means over three independent runs, with standard deviations shown as subscripts. Models are fine-tuned on the English-US training split and evaluated on the English-US test split. EM denotes exact semantic-frame match.}
	\label{tab:main_massive}
\resizebox{\linewidth}{!}{%
\begin{tabular}{lcccccccccc}
	\toprule
	\multirow{2}{*}{\textbf{Method}}
	& \multirow{2}{*}{\makecell{\textbf{Deployment}\\\textbf{bits}}}
	& \multicolumn{3}{c}{\textbf{Qwen3-0.6B}}
	& \multicolumn{3}{c}{\textbf{Llama-3.2-1B-Inst.}} 
	& \multicolumn{3}{c}{\textbf{Llama-3.2-3B-Inst.}} \\
	\cmidrule(lr){3-5}
	\cmidrule(lr){6-8}
	\cmidrule(lr){9-11}
	&
	& \textbf{EM} & \textbf{Intent} & \textbf{Slot F1}
	& \textbf{EM} & \textbf{Intent} & \textbf{Slot F1} 
	& \textbf{EM} & \textbf{Intent} & \textbf{Slot F1}\\
	\midrule
	Base-16        & 16      & \stat{0.00}{0.00} & \stat{0.00}{0.00} & \stat{3.38}{0.00}& \stat{0.00}{0.00} & \stat{0.13}{0.00} & \stat{1.39}{0.00} & \stat{0.07}{0.00} & \stat{3.03}{0.00} & \stat{3.63}{0.00}  \\
	16-bit SFT         & 16       & \stat{57.92}{0.36}& \stat{83.95}{0.22}& \stat{65.31}{0.44}& \stat{63.06}{0.41}& \stat{85.57}{0.29}& \stat{73.39}{0.38} & \stat{66.16}{0.33}& \secondstat{88.97}{0.24}& \stat{75.21}{0.35} \\
	\midrule
	QLoRA-16A      & Mixed (4.08)    & \stat{57.16}{0.42}& \stat{83.49}{0.31}& \stat{66.74}{0.48} & \stat{63.39}{0.37} & \stat{85.81}{0.26} & \stat{74.13}{0.43} & \stat{65.23}{0.39}& \stat{88.06}{0.28} & \stat{75.14}{0.41}  \\
	QA-LoRA          & Mixed (4.05) & \stat{50.17}{0.58}& \stat{80.77}{0.47} & \stat{61.19}{0.62} & \stat{54.10}{0.55} & \stat{84.30}{0.35} & \stat{65.41}{0.66} & - & - &- \\
	\midrule\midrule
	Base-4        & \textbf4      & \stat{0.00}{0.00} & \stat{0.61}{0.00} & \stat{3.97}{0.00} & \stat{0.00}{0.00} & \stat{0.03}{0.00} & \stat{0.62}{0.00}& \stat{0.00}{0.00} & \stat{1.95}{0.00} & \stat{3.32}{0.00} \\
	QLoRA-4Merge          & \textbf4     & \stat{0.00}{0.00} & \stat{0.67}{0.11} & \stat{3.61}{0.18} & \stat{0.00}{0.00} & \stat{0.03}{0.02} & \stat{0.48}{0.14} & \stat{0.00}{0.00} & \stat{1.95}{0.13} & \stat{3.29}{0.21}\\
	LoQT           & \textbf4    & \stat{51.42}{0.63}& \stat{72.09}{0.51}& \stat{45.83}{0.78}& \stat{52.14}{0.59} & \stat{75.27}{0.46} & \stat{49.90}{0.82}& - & - &- \\
	L4Q            & \textbf4     & \stat{48.27}{0.71} & \stat{72.03}{0.56} & \stat{46.02}{0.74} & \stat{53.43}{0.64} & \stat{75.88}{0.49} & \stat{58.62}{0.69} & - & - &-\\
	PV-Tuning      & \textbf4     & \stat{58.71}{0.45} & \stat{82.92}{0.37} & \stat{66.29}{0.53} & \stat{64.50}{0.38} & \stat{86.31}{0.31} & \stat{74.88}{0.46} & \stat{65.89}{0.35} & \stat{86.95}{0.27} & \stat{76.20}{0.42}  \\
	QuZO           & \textbf4     & \stat{31.26}{0.86} & \stat{57.30}{0.72} & \stat{38.55}{0.91} & \stat{46.35}{0.77} & \stat{64.17}{0.68} & \stat{46.80}{0.88} & - & - &- \\
	QZO            & \textbf4     & \stat{30.06}{0.79} & \stat{59.79}{0.64} & \stat{42.50}{0.83} & \stat{45.18}{0.81} & \stat{62.50}{0.74} & \stat{41.33}{0.95} & - & - &-\\
	QES            & \textbf4     & \stat{29.25}{0.93} & \stat{52.53}{0.81} & \stat{34.41}{1.02} & \stat{43.56}{0.85} & \stat{58.35}{0.79} & \stat{39.81}{0.97}  & - & - &-\\ \midrule
	\textbf{\method{}}  (LoRA)      & \textbf4     & \secondstat{65.88}{0.34} & \secondstat{85.53}{0.28} & \beststat{76.60}{0.39} & \secondstat{67.37}{0.31} & \stat{86.89}{0.33} & \secondstat{76.25}{0.37} &\beststat{70.48}{0.27} & \beststat{90.97}{0.19} & \beststat{79.85}{0.31} \\
	\textbf{\method{}} (Full)      & \textbf4     & \beststat{66.61}{0.32} & \beststat{86.85}{0.24} & \secondstat{75.88}{0.41} & \beststat{69.17}{0.28} & \beststat{88.13}{0.21} & \beststat{78.74}{0.34}& \secondstat{67.26}{0.29} & \stat{87.53}{0.25} & \secondstat{78.24}{0.36}  \\
	\bottomrule
\end{tabular}
}
\end{table}

\subsection{Robustness across Data, Parameterizations, and Datatypes}

\paragraph{Data size and code parameterization.}\label{sec:size}
The comparison with 16-bit SFT suggests that the restricted code-space parameterization may provide useful regularization when fine-tuning data are limited. We treat this as an empirical hypothesis rather than a consequence of the local optimization theory. Table~\ref{tab:datasize} examines the hypothesis by comparing the original 7.0k-example GSM8K training set with a 167k-example training set formed by adding 160k MetaMathQA examples~\citep{yu2024metamath}.
                                                                                                                                                                
\begin{table}[ht]
	\centering
	\small
		\caption{GSM8K test accuracy (\%) across training-set sizes and parameterizations on Llama-3.2-1B-Instruct. Values are means over three independent runs, with standard deviations shown as subscripts.}
	\label{tab:datasize}
	\begin{tabular}{llcc}
		\toprule
		\textbf{Method}
		& \textbf{Parameterization}
		& \textbf{7.0k examples}
		& \textbf{167k examples}\\
		\midrule
		16-bit SFT&LoRA (rank 8)
		& \secondstat{38.9}{0.6}& \beststat{55.9}{0.7} \\
		\midrule
		\multirow{2}{*}{\method}
		&LoRA (rank 8)
		&\stat{40.5}{0.6}& \stat{54.3}{0.6} \\
		&Full matrix
		&  \beststat{41.6}{0.5}& \secondstat{54.5}{0.5}  \\
		
		\bottomrule
	\end{tabular}

\end{table}

With 7.0k examples, \method{} (LoRA) performs best; with 167k examples, 16-bit SFT becomes strongest. This result is consistent with, but does not prove, a regularization benefit from the low-rank code parameterization in the smaller-data regime. 

\paragraph{Quantization datatype.}

Based on the unified datatype-specific recovery map in \eqref{eq:recover}, \method{} does not rely on a single quantization datatype. Different datatypes change the codebook $\mathcal C$ and scale grouping, while the search rule requires only ordered code levels and the local effective-step map in Definition~1. Table~\ref{tab:datatype} shows that \method{} can be applied to NF4, INT4, and MXFP4. NF4 and MXFP4 are stronger than INT4 in this setting, and all three variants remain competitive with the adaptive PV-Tuning baseline. These results establish compatibility across the tested datatype-specific geometries without attributing the performance differences to any single codebook property.

\begin{table}[ht]
	\centering
	\small
		\caption{Results of \method{} across quantization datatypes on Llama-3.2-1B-Instruct. Values are means over three independent runs, with standard deviations shown as subscripts.}
	\label{tab:datatype}
	\begin{tabular}{llccc}
		\toprule
		\textbf{Method}
		& \textbf{Datatype}
		& \textbf{GSM8K ACC.}
		& \textbf{AlpacaEval WR.}
		& \textbf{MASSIVE EM.}\\
		\midrule
		\multirow{3}{*}{\method}
		& NF4
		& \beststat{41.6}{0.5}& \stat{56.3}{1.8}& \beststat{69.2}{0.3} \\
		&INT4
		& \stat{38.9}{0.6} & \stat{52.0}{2.4} & \stat{65.1}{0.5}  \\
		&MXFP4
		& \stat{41.1}{0.4} & \beststat{57.1}{2.2} & \stat{68.8}{0.4} \\
		\midrule
		PV-Tuning&Adaptive
		& \stat{36.9}{0.5} & \stat{52.6}{1.8} & \stat{64.5}{0.4}  \\
		\bottomrule
	\end{tabular}

\end{table}

\subsection{Search Efficiency and Component Analysis}\label{sec:efficiency}

\paragraph{End-to-end efficiency.}
Figure~\ref{fig:curve} compares wall-clock time against the best test accuracy reached so far. \method{} is not the cheapest possible step: it uses one backward pass to build the guided reference and then evaluates multiple discrete candidates. The benefit is that these evaluations are targeted. The baselines improve quickly but plateau below the final \method{} accuracy, while \method{} reaches a stronger result within the displayed time range. This supports the intended trade-off of spending one gradient computation to reduce wasted discrete search.

\begin{figure}[ht]
	\vspace{-5pt}
	\centering
	\subfigure[Wall-clock time versus the best test accuracy reached so far, evaluated once per epoch.\label{fig:curve}]{
		\includegraphics[width=0.48\textwidth]{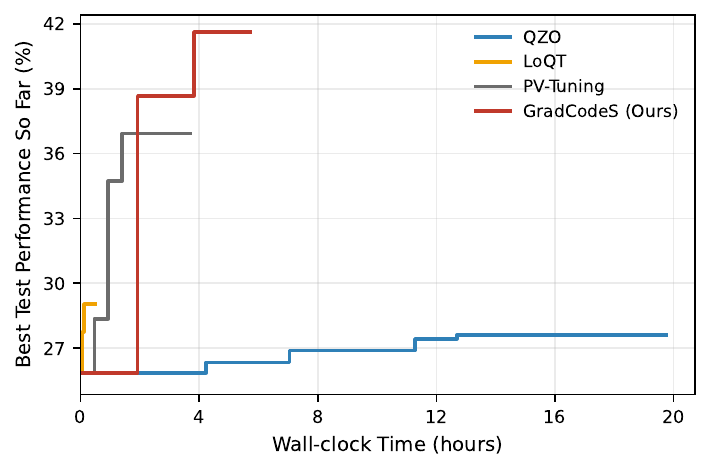}}
	\hfill
	\subfigure[Candidate loss versus distance from the guided reference in one sampling step.\label{fig:distrib}]{
		\includegraphics[width=0.48\textwidth]{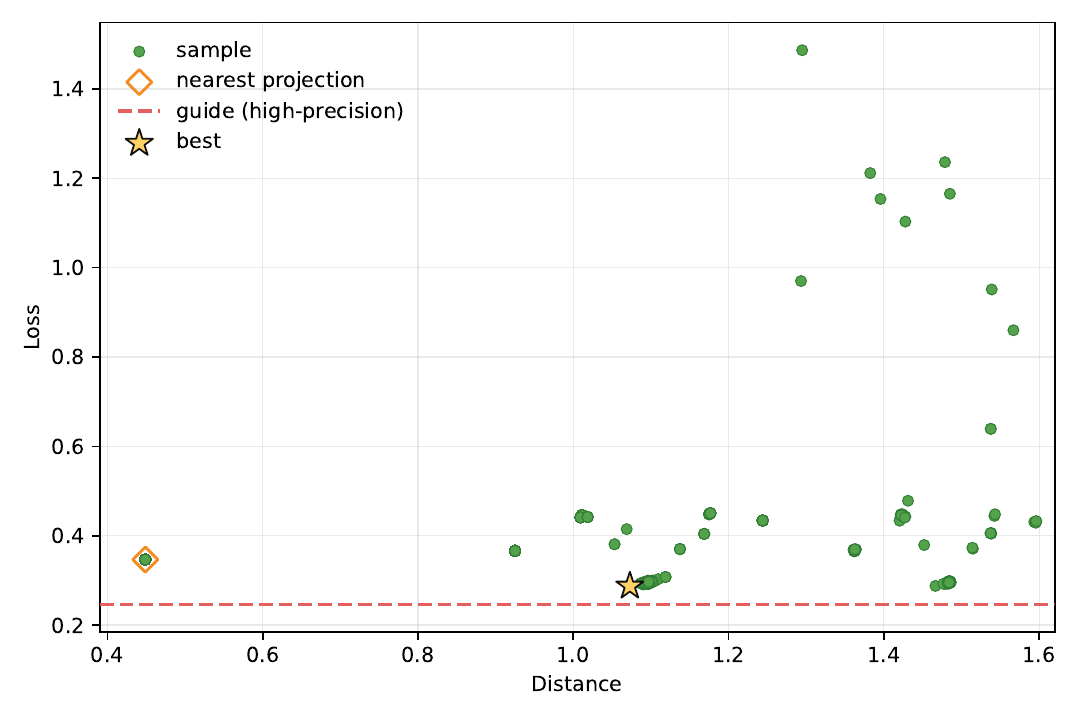}}
	\caption{Training efficiency and candidate-set quality for Llama-3.2-1B-Instruct on GSM8K.}
\end{figure}
\FloatBarrier

\paragraph{Candidate-set quality.}
Figure~\ref{fig:distrib} gives a candidate-level view of one illustrative sampling step. Candidate loss tends to be lower near the guided reference, but the relationship is imperfect. This is why \method{} does not simply project to the nearest code point: it keeps several nearby candidates reachable and lets their realized discrete losses determine the accepted update.

\begin{wraptable}{r}{0.56\textwidth}
	\vspace{-8pt}
	\centering
	\small
	\caption{Candidate-selection statistics for different sampling distributions under the same candidate budget. Values are means over three independent runs, with standard deviations shown as subscripts.}
	\label{tab:distrib}
	\begin{tabular}{lcc}
		\toprule
		\textbf{Sampling}
		& \textbf{Sel. Prob.}
		& \textbf{Sel. Mean Loss}\\
		\midrule
		Gaussian& {5.83}\%& {0.2955}  \\
		1-Hop Neighbor&{38.16}\%& {0.2963}  \\
		$\nabla_{\hat{W}} L$-Guided& {48.62}\%& {0.2951}\\
		$\nabla_{Z} L$-Guided (Ours)& {67.97}\%& {0.2949}\\
		\bottomrule
	\end{tabular}
\end{wraptable}

Table~\ref{tab:distrib} reports the empirical selection probability and selected mean loss under the same candidate budget. The guided sampler is selected substantially more often than Gaussian or one-hop sampling, indicating that gradient guidance improves candidate-set quality rather than merely increasing the number of forward evaluations.
And the proposed code surrogate gradient ($\nabla_{Z} L$)-guided sample better candidates than mapped ordinary weight gradient ($\nabla_{\hat{W}}  L$)-guided, indicating it is a more effective first order signal.
 Together, Figure~\ref{fig:distrib} and Table~\ref{tab:distrib} are consistent with the proposal-guidance role analyzed in Proposition~\ref{prop:expected-code-descent}, without directly testing its candidate-count dependence.

\paragraph{Component ablations.}
We ablate three components of the default \method{} configuration: gradient guidance, code-index updates, and scale updates. The \textit{w/o guide} variant samples uniformly from a neighborhood of the current code without conditioning on \eqref{eq:guided-reference-z}.
The $\nabla_{\hat{W}}  L$-\textit{guide} replace code surrogate gradient with mapped ordinary weight gradient.
 The \textit{w/o scale update} variant freezes group-wise scales and optimizes only codes, while the \textit{w/o code update} variant freezes the indices and optimizes only group-wise scales.

\begin{figure}[ht]
	\vspace{-5pt}
	\centering
	\includegraphics[width=0.80\textwidth]{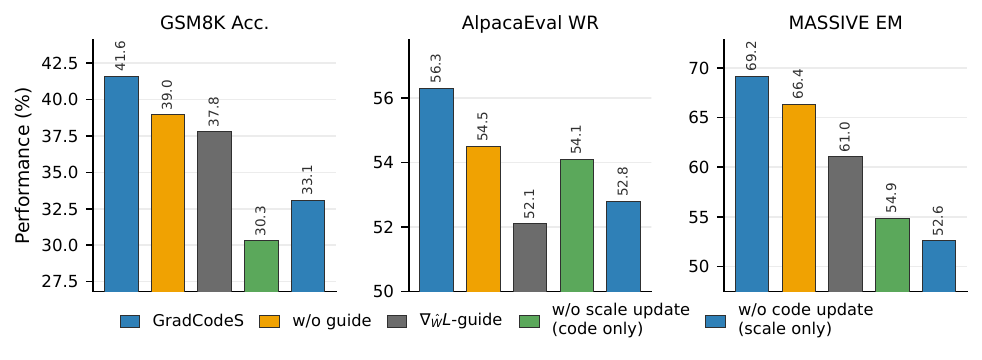}
	\caption{Component ablations for \method{} (LoRA) on Llama-3.2-1B-Instruct.}
	\label{fig:ablation}
\end{figure}

Figure~\ref{fig:ablation} shows that all three components contribute. Both removing gradient guidance and replacing $\nabla_Z L$ signal with $\nabla_{\hat{W}}  L$ consistently reduce performance, supporting the use of the code surrogate gradient to shape the proposal. Freezing either codes or scales causes larger task-dependent drops, indicating that both variables are important to the complete optimization procedure.

\section{Conclusion}

We presented \method, a gradient-guided code-search method for fine-tuning low-bit model. The key idea is to analyze quantization geometry to find the steepest descent direction in code space, and use such first-order information to shape a candidate-sampling distribution, while letting the realized discrete loss decide which state to select to be deployment-faithful. 
Across reasoning, instruction-following, and structured language-understanding tasks, \method{} consistently improves fully 4-bit adaptation and remains compatible with different codebooks, low-rank code parameterizations, and alternating scale updates. The local analysis further clarifies why guided best-of-$M$ code search can dominate both blind discrete search and single-point projected updates. An important next step is to scale the method to larger backbones and richer structured candidate-sampling families while preserving the same deployment-faithful update principle.
\clearpage


\bibliographystyle{iclr2027_conference}
\bibliography{vf}

\clearpage
\appendix

\section{Experimental Details}
\label{app:experimental-details}

\subsection{Models and Evaluation Tasks}

\paragraph{Models.}
We evaluate three instruction-tuned decoder-only language models, Qwen3-0.6B and Llama-3.2-1B and 3B-Instruct. Compact backbones are a useful stress test for fully low-bit adaptation: they are realistic for mobile, edge, and other memory-constrained deployment scenarios, and they make it easier to attribute performance differences to the optimization method rather than to very large overparameterized capacity. For each backbone, we compare the original 16-bit checkpoint, the directly quantized 4-bit checkpoint, mixed-precision adapter baselines, and fully 4-bit fine-tuned variants under the same data and evaluation protocol.

\paragraph{Datasets and metrics.}
We evaluate on three datasets covering arithmetic reasoning, general instruction following, and mobile-style language understanding. For \textbf{GSM8K}, models are fine-tuned on the official training split and evaluated on the official test split; we report answer accuracy after extracting the final numerical answer. For \textbf{Alpaca}, models are fine-tuned on Alpaca Cleaned and evaluated with AlpacaEval; we report win rate under the same evaluator for all methods. For \textbf{MASSIVE en-US}, models are fine-tuned on the English-US training split and evaluated on the English-US test split. We formulate MASSIVE as structured semantic parsing and report Exact Match, Intent Accuracy, and Slot F1.

\subsection{Quantization Protocol}

Following the block-wise setup in Section~\ref{sec:preliminaries}, we quantize all linear parameters in the transformer blocks, including the attention and MLP projection matrices in every layer, while leaving the token embedding and LM head unquantized, following common LLM quantization practice~\citep{dettmers2023qlora,frantar2023gptq,lin2024awq}. We use the same target modules, quantization format, training data, and evaluation scripts whenever applicable. The reported 4-bit scores reflect the actually deployable checkpoint rather than the relaxed training state.

\subsection{Baselines}

We compare with three groups of baselines. A method is considered \textbf{fully low-bit} deployable only if inference does not require high-precision adapters or residual branches for the quantized modules. A method is considered \textbf{mixed} if it uses a quantized backbone but keeps high-precision inference-time adapter parameters. We also include \textbf{unquantized} references to separate adaptation quality from deployment cost.

\textit{Unquantized references.}
\begin{itemize}[leftmargin=*,nosep]
	\item \textbf{Base-16} evaluates the original 16-bit checkpoint without task fine-tuning.
	\item \textbf{16-bit SFT} fine-tunes 16-bit LoRA adapters on the 16-bit backbone and serves as a high-precision parameter-efficient adaptation reference~\citep{hu2022lora}.
\end{itemize}

\textit{Mixed-precision baselines.}
\begin{itemize}[leftmargin=*,nosep]
	\item \textbf{QLoRA-16A} freezes a 4-bit backbone and fine-tunes a 16-bit LoRA adapter, giving a strong low-memory training baseline whose deployed model still contains high-precision adapter weights~\citep{dettmers2023qlora}.
	\item \textbf{QA-LoRA} uses quantization-aware group-wise low-rank adapters to improve compatibility between LoRA adaptation and quantized deployment, but optimizes continuous adapter variables~\citep{xu2023qalora}.
\end{itemize}

\textit{Fully 4-bit baselines.}
\begin{itemize}[leftmargin=*,nosep]
	\item \textbf{Base-4} directly evaluates the quantized 4-bit checkpoint before adaptation.
	\item \textbf{QLoRA-4Merge} follows QLoRA training, then merges the learned adapter into the backbone and requantizes the merged weights to 4-bit for inference~\citep{dettmers2023qlora}.
	\item \textbf{LoQT} trains low-rank adapters for quantized weights and periodically merges them into quantized full-rank weights; we evaluate its final merged-and-requantized checkpoint~\citep{loeschcke2024loqt}.
	\item \textbf{L4Q} is a representative projection/STE-style low-bit fine-tuning baseline~\citep{jeon2025l4q}.
	\item \textbf{PV-Tuning} optimizes a compressed representation and alternates quantized variables with auxiliary continuous quantities~\citep{malinovskii2024pvtuning}.
	\item \textbf{QuZO}, \textbf{QZO}, and \textbf{QES} represent deployment-faithful discrete or zeroth-order low-bit optimization baselines~\citep{zhou2025quzo,shang2025qzo,xu2026qes}.
\end{itemize}

Our method, \textbf{\method}, searches over low-bit code states directly. Every selected update is a deployable low-bit state, so the final model does not require a 16-bit adapter or high-precision residual branch at inference time. We evaluate both the full-matrix code update, \method{} (Full), and the low-rank code parameterization from Appendix~\ref{sec:lora-code}, \method{} (LoRA).

\section{Datatype-Specific Instantiations of the Unified Quantization Model}
\label{app:datatype}

The common update rule depends on each datatype only through its ordered codebook, scale grouping, and recovery map. Table~\ref{tab:datatype-details} summarizes these components for the evaluated datatypes.

\begin{table}[ht]
	\centering
	\small
	\renewcommand{\arraystretch}{1.15}
	\begin{tabular}{lcccc}
		\toprule
		Datatype
		& $b_q$
		& $K_q$
		& Codebook
		& Scale/group structure \\
		\midrule
		NF4
		& 4
		& 16
		& Non-uniform
		& Group-wise scaling \\
		INT4
		& 4
		& 16
		& Uniform
		& Group-wise scaling \\
		MXFP4
		& 4
		& 16
		& E2M1-style
		& Block/microscaling \\
		\bottomrule
	\end{tabular}
	\caption{Datatype-specific components used by the unified quantization representation.}
	\label{tab:datatype-details}
\end{table}

For datatype $q$, let $\mathcal C_q=\{c^{(q)}_1,\ldots,c^{(q)}_{K_q}\}$, let $g_q(i,j)$ denote the scale group of entry $(i,j)$, and let $S_q(s)=[s_{g_q(i,j)}]_{ij}$. The deployed weight is
\begin{equation}
	\widehat W_q(Z,s)=S_q(s)\odot\mathcal C_q(Z),
	\label{eq:datatype-recover}
\end{equation}
with datatype-specific code and scale constraints. A feasible code move changes an entry by
\begin{equation*}
	\Delta\widehat W_{q,ij}
	=s_{g_q(i,j)}\!\left(c^{(q)}_{Z_{ij}+\Delta Z_{ij}}-c^{(q)}_{Z_{ij}}\right).
\end{equation*}
Thus the effective step automatically incorporates the datatype-specific codebook gap and scale. Here \emph{datatype-agnostic} refers to the common optimization rule, not to identical feasible spaces; every candidate is still evaluated through the appropriate recovery map in \eqref{eq:datatype-recover}.

\section{Derivation of the Code Surrogate Gradient}
\label{app:code-gradient-derivation}

The chain rule would suggest
$\nabla_ZL=\nabla_{\widehat W}L\odot\nabla_Z\widehat W$, but the codebook lookup has no standard derivative. For an actual code-coordinate move $\Delta Z\in\mathbb Z^{d_{\mathrm{out}}\times d_{\mathrm{in}}}$, the deployed-weight change is exactly
\[
\widehat W(\mathrm{clip}(Z+\Delta Z,1,K),s)-\widehat W(Z,s)
=
S(s)\odot\Bigl(\mathcal C(\mathrm{clip}(Z+\Delta Z,1,K))-\mathcal C(Z)\Bigr).
\]
We therefore extend $\mathcal C(\cdot)$ piecewise linearly around the current code, using the adjacent gaps $G^+(Z)$ and $G^-(Z)$ from Section~\ref{sec:code-gradient} as directional slopes. For $\|\Delta Z\|_\infty\le 1$, this continuation satisfies
\[
\widehat W(\mathrm{clip}(Z+\Delta Z,1,K),s)-\widehat W(Z,s)
=
S(s)\odot\Bigl(G^+(Z)\odot(\Delta Z)_+
-G^-(Z)\odot(\Delta Z)_-\Bigr).
\]
For a descent-aligned move, a negative weight-gradient entry calls for a forward code move, whereas a positive entry calls for a backward code move. Writing $\sigma=\operatorname{sign}(\nabla_{\widehat W}L)$, the corresponding nonnegative directional slope is
\[
S(s)\odot\Bigl(G^+(Z)\odot\sigma_-+G^-(Z)\odot\sigma_+\Bigr),
\]
which is precisely the effective step $D\bigl(Z,s,\nabla_{\widehat W}L\bigr)$ in Definition~1. Applying the chain rule to this local continuation yields
\[
\nabla_ZL(Z,s)
=
\nabla_{\widehat W}L(Z,s)
\odot
D\bigl(Z,s,\nabla_{\widehat W}L(Z,s)\bigr),
\]
as stated in \eqref{eq:glat}. For a uniform codebook and interior coordinates, the adjacent gaps coincide, so $D$ reduces to the quantization gap multiplied by $S(s)$.

\section{Proofs for Section~\ref{sec:code-properties}}
\label{app:code-gradient}

\paragraph{Restatement of Lemma~\ref{lem:code-local-surrogate}.}
Define the local directional surrogate function of $L(Z,s)$ by extending it to real-valued interpolation $\Delta Z\in\mathbb R^{d_{\mathrm{out}}\times d_{\mathrm{in}}}$ by
\[
\bar{L}_{Z,s}(\Delta Z)
=
L\!\left(
\widehat W(Z,s)
+
D\bigl(Z,s,\nabla_{\widehat W}L(Z,s)\bigr)\odot\Delta Z
\right).
\]
Then
\[
\left.
\nabla_{\Delta Z}
\bar{L}_{Z,s}(\Delta Z)
\right|_{\Delta Z=0}
=
\nabla_Z L(Z,s).
\]
Consequently, whenever $\nabla_Z L(Z,s)\neq 0$, the direction $-\nabla_Z L(Z,s)$ is the steepest descent direction of $\bar{L}_{Z,s}$ at $\Delta Z=0$ under the Frobenius norm.

\begin{proof}[Proof of Lemma~\ref{lem:code-local-surrogate}]
	By definition,
	\[
	\bar{L}_{Z,s}(\Delta Z)
	=
	L\!\left(
	\widehat W(Z,s)
	+
	D\bigl(Z,s,\nabla_{\widehat W}L(Z,s)\bigr)\odot\Delta Z
	\right).
	\]
	Since $D\bigl(Z,s,\nabla_{\widehat W}L(Z,s)\bigr)$ is fixed at the current point, the map
	\[
	\Delta Z\mapsto
	\widehat W(Z,s)
	+
	D\bigl(Z,s,\nabla_{\widehat W}L(Z,s)\bigr)\odot\Delta Z
	\]
	is affine, and its differential in direction $H$ is
	\[
	D\bigl(Z,s,\nabla_{\widehat W}L(Z,s)\bigr)\odot H.
	\]
	Therefore, by the chain rule,
	\[
	\nabla_{\Delta Z} \bar{L}_{Z,s}(\Delta Z)
	=
	\nabla_{\widehat W}
	L\!\left(
	\widehat W(Z,s)
	+
	D\bigl(Z,s,\nabla_{\widehat W}L(Z,s)\bigr)\odot\Delta Z
	\right)
	\odot
	D\bigl(Z,s,\nabla_{\widehat W}L(Z,s)\bigr).
	\]
	Evaluating at $\Delta Z=0$ gives
	\[
	\begin{aligned}
		&\left.
		\nabla_{\Delta Z} \bar{L}_{Z,s}(\Delta Z)
		\right|_{\Delta Z=0}
		\\
		&=
		\nabla_{\widehat W} L\bigl(\widehat W(Z,s)\bigr)
		\odot
		D\bigl(Z,s,\nabla_{\widehat W}L(Z,s)\bigr)
		\\
		&=
		\nabla_{\widehat W}L(Z,s)\odot D\bigl(Z,s,\nabla_{\widehat W}L(Z,s)\bigr)
		=
		\nabla_Z L(Z,s),
	\end{aligned}
	\]
	which proves the exact-gradient claim.
	
	For the steepest-descent statement, consider the first-order model
	\[
	m(\Delta Z)
	=
	\bar{L}_{Z,s}(0)+\langle \nabla_Z L(Z,s),\Delta Z\rangle_F .
	\]
	For any $\Delta Z$ with $\|\Delta Z\|_F\le r$, Cauchy--Schwarz gives
	\[
	\langle \nabla_Z L(Z,s),\Delta Z\rangle_F
	\geq
	-\|\nabla_Z L(Z,s)\|_F\,\|\Delta Z\|_F
	\geq
	-r\|\nabla_Z L(Z,s)\|_F.
	\]
	Equality is attained at
	\[
	\Delta Z
	=
	-r\,
	\frac{\nabla_Z L(Z,s)}{\|\nabla_Z L(Z,s)\|_F},
	\]
	whenever $\nabla_Z L(Z,s)\neq 0$. Hence $-\nabla_Z L(Z,s)$ is the steepest descent direction of $\bar{L}_{Z,s}$ at $\Delta Z=0$ under the Frobenius norm.
\end{proof}

\paragraph{Restatement of Proposition~\ref{prop:code-one-step-consistency}.}
Let $\Delta Z\in\{-1,0,1\}^{d_{\mathrm{out}}\times d_{\mathrm{in}}}$ be a feasible one-step code move such that
\[
Z+\Delta Z\in\mathcal Z,
\qquad
\Delta Z\odot\nabla_{\widehat W}L(Z,s)\le 0
\quad\text{entrywise}.
\]
Then the surrogate linear term matches the exact first-order deployed-weight change:
\[
\langle \nabla_Z L(Z,s),\Delta Z\rangle_F
=
\left\langle
\nabla_{\widehat W}L(Z,s),
\widehat W(Z+\Delta Z,s)-\widehat W(Z,s)
\right\rangle_F .
\]

\begin{proof}[Proof of Proposition~\ref{prop:code-one-step-consistency}]
	We first prove the exact first-order identity
	\[
	\langle \nabla_Z L(Z,s),\Delta Z\rangle_F
	=
	\left\langle
	\nabla_{\widehat W}L(Z,s),
	\widehat W(Z+\Delta Z,s)-\widehat W(Z,s)
	\right\rangle_F.
	\]
	Fix one coordinate $(i,j)$. If $\Delta Z_{ij}=0$, both sides contribute zero. If $\Delta Z_{ij}=1$, feasibility gives $Z_{ij}<K$, and gradient alignment gives $(\nabla_{\widehat W}L(Z,s))_{ij}\le 0$. When $(\nabla_{\widehat W}L(Z,s))_{ij}<0$, the definition of $D\bigl(Z,s,\nabla_{\widehat W}L(Z,s)\bigr)$ selects the forward gap, so
	\[
	\begin{aligned}
	\widehat W_{ij}(Z+\Delta Z,s)-\widehat W_{ij}(Z,s)
	&=S_{ij}(s)\bigl(c_{Z_{ij}+1}-c_{Z_{ij}}\bigr)\\
	&=d^+_{ij}(Z,s)\\
	&=\bigl[D\bigl(Z,s,\nabla_{\widehat W}L(Z,s)\bigr)\bigr]_{ij}\Delta Z_{ij}.
	\end{aligned}
	\]
	If instead $(\nabla_{\widehat W}L(Z,s))_{ij}=0$, then both first-order contributions are zero. If $\Delta Z_{ij}=-1$, feasibility gives $Z_{ij}>1$, and gradient alignment gives $(\nabla_{\widehat W}L(Z,s))_{ij}\ge 0$. When $(\nabla_{\widehat W}L(Z,s))_{ij}>0$, the definition of $D\bigl(Z,s,\nabla_{\widehat W}L(Z,s)\bigr)$ selects the backward gap, so
	\[
	\begin{aligned}
	\widehat W_{ij}(Z+\Delta Z,s)-\widehat W_{ij}(Z,s)
	&=S_{ij}(s)\bigl(c_{Z_{ij}-1}-c_{Z_{ij}}\bigr)\\
	&=-d^-_{ij}(Z,s)\\
	&=\bigl[D\bigl(Z,s,\nabla_{\widehat W}L(Z,s)\bigr)\bigr]_{ij}\Delta Z_{ij}.
	\end{aligned}
	\]
	Again, if $(\nabla_{\widehat W}L(Z,s))_{ij}=0$, both first-order contributions are zero. Therefore, for every coordinate,
	\[
	(\nabla_{\widehat W}L(Z,s))_{ij}
	\bigl(\widehat W_{ij}(Z+\Delta Z,s)-\widehat W_{ij}(Z,s)\bigr)
	=
	(\nabla_{\widehat W}L(Z,s))_{ij}
	\bigl[D\bigl(Z,s,\nabla_{\widehat W}L(Z,s)\bigr)\bigr]_{ij}
	\Delta Z_{ij}.
	\]
	Summing over coordinates yields
	\begin{align}\nonumber
	\left\langle
\nabla_{\widehat W}L(Z,s),
\widehat W(Z+\Delta Z,s)-\widehat W(Z,s)
\right\rangle_F
=
\left\langle
\nabla_{\widehat W}L(Z,s)\odot D\bigl(Z,s,\nabla_{\widehat W}L(Z,s)\bigr),
\Delta Z
\right\rangle_F\\\nonumber
=
\langle \nabla_Z L(Z,s),\Delta Z\rangle_F,
	\end{align}
	which proves \eqref{eq:one-step-consistency}.
\end{proof}

\paragraph{Restatement of Corollary~\ref{prop:code-one-step-upper-bound}.}
Under the assumptions of Proposition~\ref{prop:code-one-step-consistency}, if $L$ is $\beta$-smooth as a function of the deployed weight in a neighborhood of $\widehat W(Z,s)$, then
\[
L(Z+\Delta Z,s)
\le
L(Z,s)
+
\langle \nabla_Z L(Z,s),\Delta Z\rangle_F
+
\frac{\beta}{2}
\left\|
\widehat W(Z+\Delta Z,s)-\widehat W(Z,s)
\right\|_F^2 .
\]

\begin{proof}[Proof of Corollary~\ref{prop:code-one-step-upper-bound}]
	If $L$ is $\beta$-smooth as a function of the deployed weight in a neighborhood of $\widehat W(Z,s)$, the descent lemma gives
\begin{align}\nonumber
	L\bigl(\widehat W(Z+\Delta Z,s)\bigr)
\le
L\bigl(\widehat W(Z,s)\bigr)
+
\left\langle
\nabla_{\widehat W} L\bigl(\widehat W(Z,s)\bigr),
\widehat W(Z+\Delta Z,s)-\widehat W(Z,s)
\right\rangle_F\\\nonumber
+
\frac{\beta}{2}
\left\|
\widehat W(Z+\Delta Z,s)-\widehat W(Z,s)
\right\|_F^2.
\end{align}
	By Proposition~\ref{prop:code-one-step-consistency},
	\[
	\left\langle
	\nabla_{\widehat W} L\bigl(\widehat W(Z,s)\bigr),
	\widehat W(Z+\Delta Z,s)-\widehat W(Z,s)
	\right\rangle_F
	=
	\langle \nabla_Z L(Z,s),\Delta Z\rangle_F.
	\]
	Substituting this identity together with
	\[
	L(Z,s)=L\bigl(\widehat W(Z,s)\bigr),
	\qquad
	L(Z+\Delta Z,s)=L\bigl(\widehat W(Z+\Delta Z,s)\bigr)
	\]
	yields
	\[
	L(Z+\Delta Z,s)
	\le
	L(Z,s)
	+
	\langle \nabla_Z L(Z,s),\Delta Z\rangle_F
	+
	\frac{\beta}{2}
	\left\|
	\widehat W(Z+\Delta Z,s)-\widehat W(Z,s)
	\right\|_F^2,
	\]
	as claimed.
\end{proof}

\section{A Concrete Factorized Candidate-Sampling Distribution}
\label{app:proposal-spec}

To instantiate the local candidate-sampling rule in Section~\ref{sec:guided}, define the neighborhood around the guided reference $Z^{\mathrm{ref}}$ by
\begin{equation*}
	\mathcal N^{(\rho)}(Z^{\mathrm{ref}})
	=
	\left\{
	Z' \in \mathcal Z \mid
	\|Z'-Z^{\mathrm{ref}}\|_\infty \le \rho
	\right\},
\end{equation*}
where $\rho\in\mathbb Z^+$ is the code search radius.
For efficient sampling in high-dimensional spaces, we sample each matrix entry $(i,j)\in[1,d_{\text{out}}]\times[1,d_{\text{in}}]$ independently from a discrete scalar distribution $p_{ij}(u\mid Z^{\mathrm{ref}})$, which jointly induces
\[
p(Z^{m}\mid Z^{\mathrm{ref}})
=
\prod\nolimits_{i=1}^{d_{\text{out}}}\prod\nolimits_{j=1}^{d_{\text{in}}}
p_{ij}(Z^{m}_{ij}\mid Z^{\mathrm{ref}}).
\]
We choose
\begin{equation*}
	p_{ij}(u\mid Z^{\mathrm{ref}})
	=
	\frac{\phi(|u-Z^{\mathrm{ref}}_{ij}|)}
	{\sum_{v \in \mathcal N^{(\rho)}_{ij}(Z^{\mathrm{ref}})} \phi(|v-Z^{\mathrm{ref}}_{ij}|)},
	\qquad
	\mathcal N^{(\rho)}_{ij}(Z^{\mathrm{ref}})
	=
	\left\{
	u \in [K]\mid
	|u-Z^{\mathrm{ref}}_{ij}|\le\rho
	\right\}
\end{equation*}
where $\phi(x)$ is a monotonically decreasing function, such as $(x+\varepsilon)^{-1}$ with $\varepsilon>0$ or $\exp(-x^2)$, so that code states closer to $Z^{\mathrm{ref}}$ receive larger coordinatewise probability mass.

\section{Expected Descent of the Code-Search Substep}
\label{app:expected-descent}

For the proof of Proposition~\ref{prop:expected-code-descent}, we analyze the deterministic full-batch code-search substep of Algorithm~\ref{alg:code-search} with a fixed scale $s$. At iteration $t$, define the code-space surrogate gradient and continuous reference step by
\[
g_t=\nabla_Z L(Z_t,s),
\qquad
\bar{\Delta}_t=-\eta_Z g_t,
\qquad
Z_t^{\mathrm{ref}}=Z_t+\bar{\Delta}_t.
\]
For a sampled candidate state $Z_t^m$, let
\[
\Delta_t^m=Z_t^m-Z_t
\]
denote its code update relative to the current state.
The proof uses the following assumptions.

\paragraph{Assumption A1 (smoothness on admissible sampled moves).}
There exist constants $\beta>0$ and $\kappa>0$ such that, for every iteration $t$ and every candidate update $\Delta$ considered below, Corollary~\ref{prop:code-one-step-upper-bound} applies and
\begin{equation}
	\left\|
	\widehat W(Z_t+\Delta,s)-\widehat W(Z_t,s)
	\right\|_F
	\le
	\kappa \|\Delta\|_F.
	\label{eq:weight-lipschitz-delta}
\end{equation}

\paragraph{Assumption A2 (near-reference proposal mass).}
Fix some $c\in[0,1)$. For each iteration $t$, define the good-update set
\begin{equation*}
	\mathcal G_t
	=
	\Bigl\{
	\Delta
	\;\Big|\;
	\Delta \text{ is gradient-aligned and one-step feasible, and }
	\|\Delta+\eta_Z g_t\|_F
	\le
	c\,\eta_Z\|g_t\|_F
	\Bigr\}.
\end{equation*}
Let $q_t$ denote the total proposal mass of candidate states whose update belongs to $\mathcal G_t$:
\begin{equation*}
	q_t
	=
	\sum\nolimits_{Z'\,:\,Z'-Z_t\in\mathcal G_t}
	p(Z'\mid Z_t^{\mathrm{ref}}).
\end{equation*}

\begin{proof}[Proof of Proposition~\ref{prop:expected-code-descent}]
	Fix an iteration $t$. If $\Delta\in\mathcal G_t$, then by definition of $\mathcal G_t$,
	\[
	\langle g_t,\Delta\rangle
	=
	\langle g_t,-\eta_Z g_t\rangle
	+
	\langle g_t,\Delta+\eta_Z g_t\rangle
	\le
	-\eta_Z\|g_t\|_F^2
	+
	\|g_t\|_F\,\|\Delta+\eta_Z g_t\|_F
	\le
	-(1-c)\eta_Z\|g_t\|_F^2,
	\]
	and also
	\[
	\|\Delta\|_F
	\le
	\|\eta_Z g_t\|_F+\|\Delta+\eta_Z g_t\|_F
	\le
	(1+c)\eta_Z\|g_t\|_F.
	\]
	Combining this with \eqref{eq:weight-lipschitz-delta} gives
	\[
	\left\|
	\widehat W(Z_t+\Delta,s)-\widehat W(Z_t,s)
	\right\|_F^2
	\le
	\kappa^2(1+c)^2\eta_Z^2\|g_t\|_F^2.
	\]
	Therefore, by Corollary~\ref{prop:code-one-step-upper-bound},
	\[
	L(Z_t+\Delta,s)
	\le
	L(Z_t,s)
	-
	\left(
	1-c-\frac{\beta\kappa^2\eta_Z}{2}(1+c)^2
	\right)
	\eta_Z\|g_t\|_F^2
	=
	L(Z_t,s)-\alpha\eta_Z\|g_t\|_F^2.
	\]
	
	Now let
	\[
	E_t
	=
	\Bigl\{
	\text{at least one of the $M$ sampled candidates has update in }\mathcal G_t
	\Bigr\}.
	\]
	By independence of the $M$ samples,
	\[
	\Pr(E_t\mid Z_t)
	=
	1-(1-q_t)^M.
	\]
	On the event $E_t$, the sampled set contains a candidate whose loss is at most
	\(
	L(Z_t,s)-\alpha\eta_Z\|g_t\|_F^2
	\).
	On $E_t$, the best-of-$M$ rule therefore returns a candidate whose loss is no larger than this value. On the complementary event $E_t^c$, the strict accept-if-improving rule gives $L(Z_{t+1},s)\le L(Z_t,s)$. Combining the two cases yields the pointwise bound
	\[
	L(Z_{t+1},s)
	\le
	L(Z_t,s)-\alpha\eta_Z\|g_t\|_F^2 \mathbf 1_{E_t}.
	\]
	Taking conditional expectation gives
	\[
	\mathbb E\!\left[L(Z_{t+1},s)\mid Z_t\right]
	\le
	L(Z_t,s)
	-
	\alpha\,\eta_Z\bigl[1-(1-q_t)^M\bigr]\|g_t\|_F^2,
	\]
	which proves the claimed one-step descent inequality.
\end{proof}

\section{Low-Rank Code Update}
\label{sec:lora-code}
The core search rule described above only requires a code matrix $Z\in\mathcal Z$ to induce weight $\widehat W(Z)$, and a differentiable loss to compute \eqref{eq:glat}. As a result, the same rule can also be used with other parameterizations.
Low-rank parameterizations are widely used in LLM fine-tuning to reduce the number of optimization variables and improve parameter efficiency. We apply this principle in code space by parameterizing a local code update as a low-rank decomposition with two integer-valued factors:
\begin{equation*}
	\Delta Z(A,B)=AB^\top,
	\qquad
	A\in\mathbb Z^{d_{\mathrm{out}}\times R},\quad
	B\in\mathbb Z^{d_{\mathrm{in}}\times R}.
\end{equation*}
This reduces the number of searched coordinates from $d_{\mathrm{out}}d_{\mathrm{in}}$ to $R(d_{\mathrm{out}}+d_{\mathrm{in}})$. 
Unlike ordinary LoRA, $A$ and $B$ are not inference-time residual adapters; after training, only the deployed code state $\mathrm{clip}(Z+\Delta Z(A,B),1,K)$ needs to be stored. The factor-wise guidance follows from the same surrogate chain rule. Given $\nabla_Z L$ by \eqref{eq:glat} at the induced code state, we use
\begin{equation*}
	\nabla_A L = \nabla_Z L\,B,
	\qquad
	\nabla_B L = \left(\nabla_Z L\right)^\top A.
\end{equation*}
We form guided references for $A$ and $B$, sample integer-valued factor candidates, and select by the deployed loss after merging the induced code update.

\end{document}

%% file: math_commands.tex
\usepackage{amsmath,amsfonts,bm}

\def\eqref#1{equation~\ref{#1}}

\def\1{\bm{1}}

\DeclareMathAlphabet{\mathsfit}{\encodingdefault}{\sfdefault}{m}{sl}
\SetMathAlphabet{\mathsfit}{bold}{\encodingdefault}{\sfdefault}{bx}{n}

